\documentclass[conference]{IEEEtran}
\IEEEoverridecommandlockouts
\usepackage{amsmath}
\usepackage{bm}
\DeclareMathOperator*{\argmin}{argmin}
\usepackage{amssymb}
\usepackage{amsthm}
\newtheorem{theorem}{Theorem}

\newtheorem{lemma}[theorem]{Lemma}
\newtheorem{remark}{Remark}
\usepackage{graphicx}
\usepackage[labelformat=simple]{subcaption}

\usepackage{url}

\newcommand{\PHB}[1]{\noindent\textbf{#1}\hspace{.5em}} 
\newcommand{\PHM}[1]{\vspace{.2em}\noindent\textbf{#1}\hspace{.5em}} 

\newcommand{\A}{\bm{A}}
\newcommand{\D}{\bm{D}}
\newcommand{\I}{\bm{I}}
\renewcommand{\P}{\bm{P}}
\newcommand{\Q}{\bm{Q}}
\newcommand{\R}{\bm{R}}
\renewcommand{\S}{\bm{S}}
\newcommand{\U}{\bm{U}}
\newcommand{\X}{\bm{X}}
\newcommand{\W}{\bm{W}}
\newcommand{\Z}{\bm{Z}}
\renewcommand{\b}{\bm{b}}
\newcommand{\e}{\bm{e}}
\newcommand{\g}{\bm{g}}
\newcommand{\p}{\bm{p}}
\renewcommand{\u}{\bm{u}}
\renewcommand{\v}{\bm{v}}
\newcommand{\w}{\bm{w}}
\newcommand{\x}{\bm{x}}

\newcommand{\z}{\bm{z}}

\usepackage{cite}
\usepackage{algorithm}
\usepackage{algorithmic}

\usepackage{color}

\usepackage{xcolor}
\newcommand{\peng}[1]{#1}

\def\BibTeX{{\rm B\kern-.05em{\sc i\kern-.025em b}\kern-.08em
    T\kern-.1667em\lower.7ex\hbox{E}\kern-.125emX}}
\begin{document}

\title{Feature Reconstruction Attacks and Countermeasures of DNN training in Vertical Federated Learning 
}

\author{\IEEEauthorblockN{Peng Ye\IEEEauthorrefmark{1}, Zhifeng Jiang\IEEEauthorrefmark{1}, Wei Wang\IEEEauthorrefmark{1}, Bo Li\IEEEauthorrefmark{1}, Baochun Li\IEEEauthorrefmark{2}}
\IEEEauthorblockA{
\IEEEauthorrefmark{1}\textit{The Hong Kong University of Science and Technology} \\\IEEEauthorrefmark{2}\textit{University of Toronto}\\
\{pyeac, zjiangaj, weiwa, bli\}@cse.ust.hk, bli@ece.toronto.edu}

}

\maketitle

\begin{abstract}
Federated learning (FL) has increasingly been deployed, in its \emph{vertical form}, among organizations to facilitate secure collaborative training over siloed data. In vertical FL (VFL), participants hold disjoint features of the same set of sample instances. Among them, only one has labels. This participant, known as the \emph{active party}, initiates the training and interacts with the other participants, known as the \emph{passive parties}. Despite the increasing adoption of VFL, it remains largely unknown if and how the active party can extract feature data from the passive party, especially when training deep neural network (DNN) models.

This paper makes the first attempt to study the feature security problem of DNN training in VFL. We consider a DNN model partitioned between active and passive parties, where the latter only holds a subset of the input layer and exhibits some categorical features of \emph{binary values}. \peng{Using a reduction from the Exact Cover problem, we prove that reconstructing those binary features is NP-hard.} Through analysis, we demonstrate that, unless the feature dimension is exceedingly large, it remains feasible, both theoretically and practically, to launch a reconstruction attack with an efficient search-based algorithm that prevails over current feature protection techniques. To address this problem, we develop a novel feature protection scheme against the reconstruction attack that effectively misleads the search to some pre-specified random values. With an extensive set of experiments, we show that our protection scheme sustains the feature reconstruction attack in various VFL applications at no expense of accuracy loss.
\end{abstract}


\section{Introduction}
\label{sec:intro}
The sustained technological advances in machine learning (ML) have transformed many industries in a profound way. Companies in the internet, finance, retail, and healthcare industries are now building advanced ML models to enable new AI-driven applications, service models, and intelligent decision making. They require collecting a large volume of training data from diverse sources, which many find infeasible. In reality, data are usually dispersed in siloed organizations and data sharing is strictly forbidden -- it could not only raise serious privacy and security concerns, but also violate government regulations, such as CCPA~\cite{CCPA} in America, GDPR~\cite {GDPR} in Europe, and PIPEDA~\cite{PIPEDA} in Canada. Thus, ensuring data privacy is of paramount importance.

Federated learning (FL) has emerged as a new private-preserving learning paradigm to break data silos~\cite{yang2019Federated,kairouz2019advances,mcmahan2017communication}. It enables multiple parties to collaboratively train a global ML model over siloed data while preserving data privacy. FL has been increasingly deployed among companies to form a data federation. In this paper, we consider a typical application scenario referred to as \emph {vertical federated learning} (VFL)~\cite {weng2021Privacy,kairouz2019advances,fu2022Label,yang2019Federated,jiang2022Comprehensive}, in which participants own disjoint features (attributes) of the same set of sample instances, as illustrated in Figure~\ref{fig:vfl}. Only one participant has \emph{labels}, known as the \emph{active party}, and utilizes the joint feature data of its own and from the others, known as the \emph{passive parties}, to train an ML model. For example, a social network company and an online retailer can have an overlapping user base. The former has accumulated a rich set of user profiles (feature A) through its social network app, while the latter has user browsing history (feature B) and item ordering records (labels). Together, they form a joint dataset with user features vertically partitioned between the two types of participants. The online retailer, being the active party, can partner with the social network company to train a better recommendation model over the joint dataset.

\begin{figure}
  \centering
  \includegraphics[width=\columnwidth]{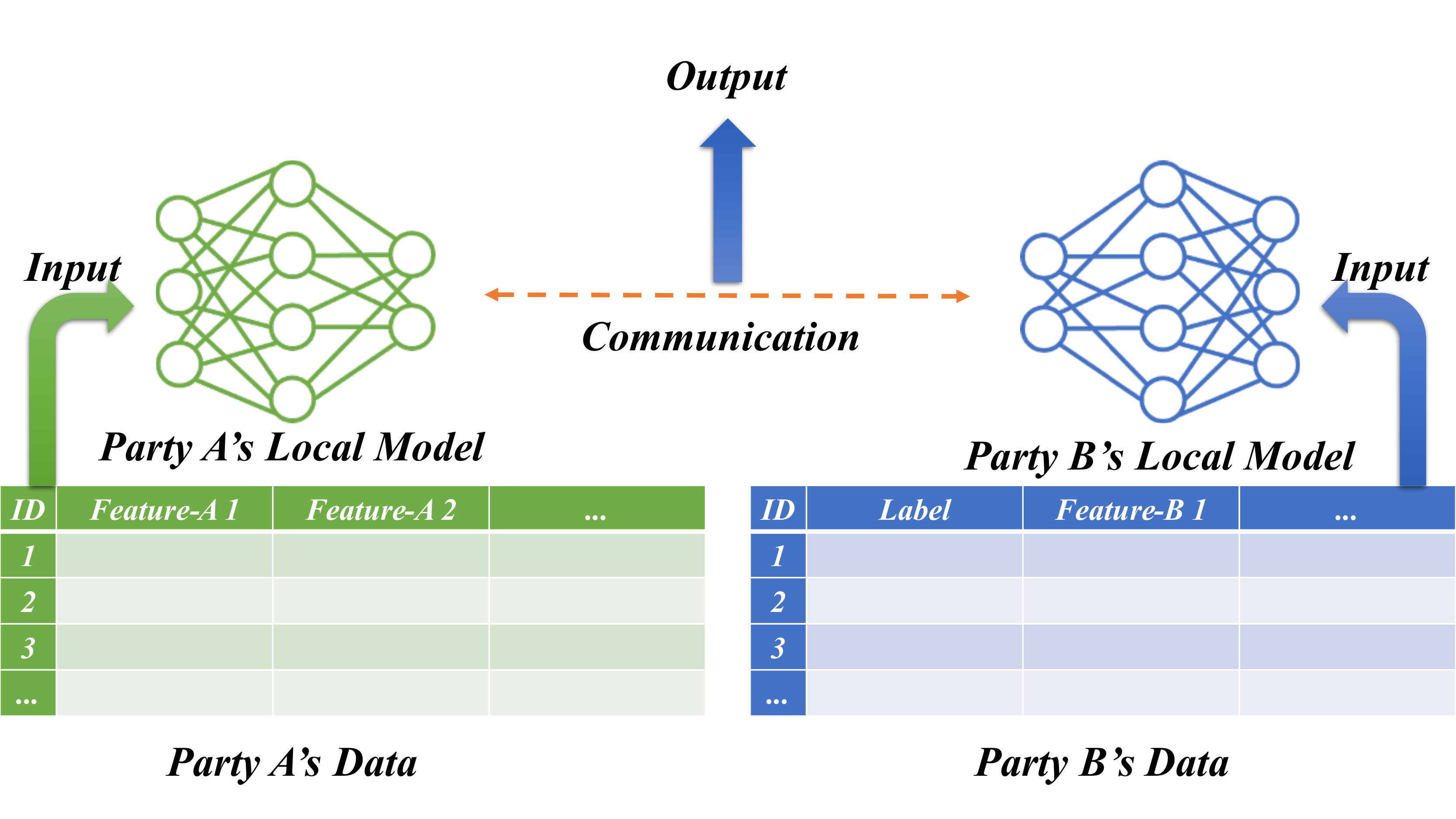}\\
  \caption{Illustration of VFL.}\label{fig:vfl}
\end{figure}

Among various models used in VFL, deep neural networks (DNN) seem to be the most promising in practice. Depending on how features are partitioned, a DNN model is split between different participants, where the passive party holds a subset of a few bottom layers, and the active party holds the rest of the neural network. Each model partition is maintained as a \emph{private local network}. The active party initiates the training and iteratively interacts with the passive party.

Clearly, raw data is not exposed in the training process; yet the intermediate results exchanged between the two parties contain a rich set of information that may reveal private information. Prior works show that private labels owned by an active party can be possibly inferred by a passive party from the received gradient updates, resulting in \emph{label leakage}~\cite{fu2022Label,li2022Label,liu2022Batch}. In this paper, we study a dual security problem concerning the \emph{feature reconstruction attack}, where an active party attempts to uncover the private features owned by a passive party. Evidently, the active party is in a more advantageous position for a feature attack with more information to exploit. However, only a few recent works have considered the feature security problem in VFL for logistic regression and decision tree models, and under rather strong assumptions. For instance, It is assumed that the adversary either has auxiliary feature data~\cite{jiang2022Comprehensive}, knows the weights of the entire model~\cite{luo2021Feature}, or can corrupt the trusted third-party~\cite{weng2021Privacy}.

To the best of our knowledge, this paper makes the first attempt to study the feature reconstruction attack of DNN training in VFL. We consider a DNN model jointly trained by two participants, where the passive party holds a subset of the input layer and the active party holds the remainder of the model. This design has a number of benefits: (1) it requires no structural adaptation of a DNN model while achieving the same accuracy as centralized training; (2) it exposes minimum attack surface for label inference~\cite{fu2022Label}; (3) it supports the state-of-the-art privacy-preserving framework~\cite{fu2022BlindFL}. We assume that the active party is an honest-but-curious adversary with no additional knowledge beyond its own data (e.g., features and labels), local models, and the intermediate results received from the passive party. We first show through analysis an \emph{impossibility theorem} (Theorem~\ref{thm:impossibility}) in that the active party cannot reconstruct \emph{general features} from the passive party that can take arbitrary values with the above knowledge only. This seems to partially explain that there have been limited explorations on feature reconstruction attacks, and each requiring some strong assumptions.

As reconstructing general features is infeasible, in this paper, we consider the attack on categorical features of \emph{binary values}\footnote{Our attack is also effective to general categorical features that take known values (Section~\ref{sec:discuss}).}, which are commonly observed in training data containing sensitive information (e.g., gender, marital and employment status). We show that the problem of binary feature reconstruction can be reduced from the Exact Cover problem (Theorem~\ref{thm:nphardness}), which is NP-hard~\cite{karp1972reducibility}. Through rigorous analysis, we show that, unless the feature dimension is exceedingly large, it remains feasible, theoretically and empirically, to reconstruct binary features with a robust search-based attack algorithm (Section~\ref{sec:attack}). We further demonstrate that such an attack cannot be effectively defended by conventional random masking approaches, nor can it be guarded by the recently proposed privacy-preserving framework for VFL training~\cite{fu2022BlindFL}.

A common defense approach is to add a mask with Gaussian noise to the intermediate output. However, as the results from the experiments will suggest, such an approach often leads to significant degradation in model accuracy. To address this problem, we propose an efficient feature protection scheme, in which we first perturb the intermediate output via a \emph{rank-reduction} technique with negligible impact on model performance. Then we insert a fabricated (randomly generated) binary feature to masquerade as the input features. We show that this can effectively lead the attacker to find the fabricated feature instead of the original input features.

We have evaluated our attack and defense methods over five datasets. Our experimental results have shown that our attack strategies can completely recover all input binary features if no protection mechanism is in place. When the intermediate results are masked by random Gaussian noise, the attack is still effective with high accuracy, but the model performance suffers when the noise is large. Our proposed defense, however, successfully misleads the attacker to a fabricated binary feature and results in nearly no loss in model performance.

\section{Background and Motivation}
\label{sec:background}
In this section, we present the background of VFL with vertically partitioned data and DNN training in VFL.

\subsection{Vertical Federated Learning}
\label{sec:vfl}
With the increasing concern on data privacy, federated learning (FL)~\cite {mcmahan2017communication,yang2019Federated,kairouz2019advances} emerges as a new paradigm for secure collaborative learning over siloed data. In FL, participants jointly build a training model without revealing the private data. Depending on how data is partitioned and applications, FL can be categorized into, \emph{vertical} FL and \emph{horizontal} FL. In this paper, we focus on vertical FL (VFL), in which participants have overlapping sample instances but own their disjoint features. In other words, participants have vertically partitioned tabular data where a sample instance is a row and a feature is a column (see Figure~\ref{fig:vfl}). Among all participants, only one participant has labels for a specific learning task, which can obtain better training performance by incorporating all features from other participants. This participant initiates VFL training, thus being the \emph{active party}, and interacts with others, i.e., the \emph{passive party}, to jointly build a model over enriched feature data.

VFL has found a wide range of applications in cross-enterprise collaboration~\cite {yang2019Federated,fu2022BlindFL,fu2022Label,luo2021Feature,weng2021Privacy}. Considering the example described earlier, the user profiles gathered by the social network company can benefit not only the business of online retailers, but also many other relevant businesses. For example, it can help a FinTech corporation build a better risk model, or a restaurant business to establish a more accurate model capturing customer dining preferences. Each of these companies can thus initiate VFL training and engages the social network company to collaborate as a passive party, similar to that depicted in Figure~\ref{fig:vfl}.

\subsection{DNN training in VFL}
VFL supports a variety of training models, ranging from regression~\cite{hardy2017private}, decision trees~\cite{cheng2021secureboost,luo2021Feature}, and to more sophisticated deep neural networks (DNN)~\cite{fu2022BlindFL,vepakomma2018split}. Among them, DNN seems to be most promising in that it achieves state-of-the-art performance in a myriad of real applications. DNN training in VFL borrows the idea from \emph{split learning}~\cite{vepakomma2018split}, in which a neural network is split into a \emph{top model} and multiple \emph{bottom models} at a certain layer called the \emph{cut layer}. All participants hold a bottom model that takes inputs of its own features; only the active party additionally holds the top model. All split models are maintained as \emph{private local models}.

In VFL training, the active party iteratively interacts with the passive party using the standard stochastic gradient descendent (SGD) algorithm, as illustrated in Figure~\ref{fig:workflow}. In the forward pass of an iteration, each participant (active and passive) computes the forward activation of its bottom model using its own features, and passes the result to the active party for aggregation. In the meanwhile, the active party also feeds the aggregated results to the top model and obtains the prediction output of the entire neural network. In the backward pass, the active party computes the gradients based on the prediction output and its own labels and sends the passive party the expected gradients w.r.t. the cut layer. The passive party continues the backward pass and computes the gradients w.r.t. the bottom model. All participants can now update the local models via gradient descent and proceed to the next iteration.

\begin{figure}
 \centering
 \includegraphics[width=\columnwidth]{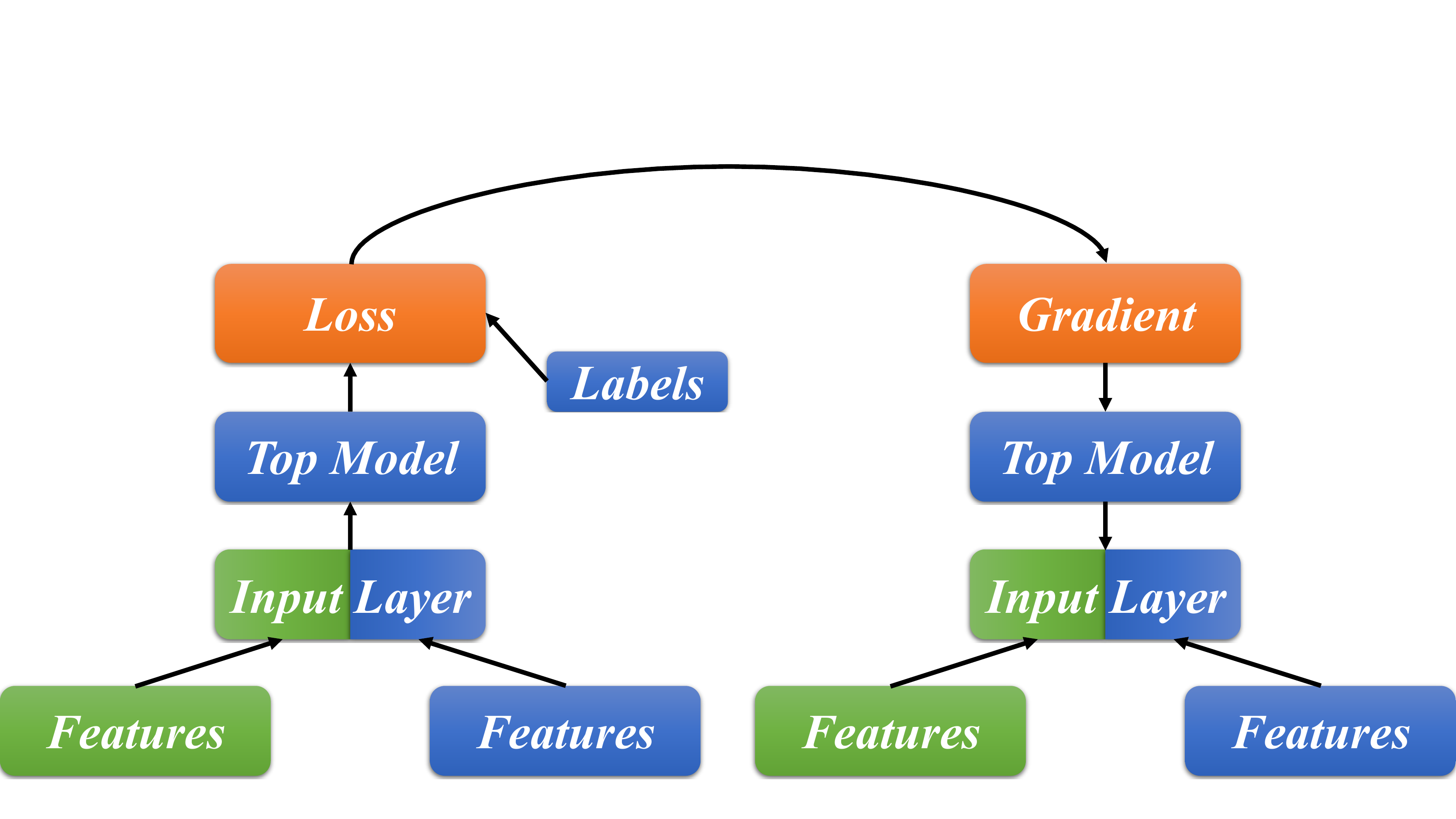}\\
 \caption{The workflow of forward pass (left) and backward pass (right) in
 a DNN training iteration of VFL. }
 \label{fig:workflow}
\end{figure}

\section{Overview}
\label{sec:overview}
In this section, we first describe our VFL framework for DNN training and introduce the threat model. We then formally prove that a feature inference attack is not possible when the attacker has zero knowledge of the data before presenting the binary feature.

\subsection{Vertical Federated Learning}
Throughout this paper, we use boldface upper case letters (e.g. $\A$) to denote matrices and boldface lower case letters (e.g. $\x$) to denote vectors. We use $\bm{0}$ to denote the zero vector. Vectors are by default column vectors while row vectors will be denoted by the transpose of column vectors (e.g. $\x^\top$). The $i$-th coordinate of vector $\x$ is denoted by $\x_i$. We use $[n]$ to represent the set $\{1, 2, \dots, n\}$ for positive integer $n$. The notation $\{0, 1\}^n$ denotes the set of all $n$-dimensional binary vectors (i.e. $\{\x\in\mathbb{R}^n: \x_i\in\{0, 1\}\text{ for all } i\in[n]\}$).

Now we will describe the overall workflow of DNN training in VFL. We consider split learning~\cite{vepakomma2018split}, where the \emph{cut layer} is the first layer, i.e., the input layer. In this setting, the parameters in the input layer are divided into two parts: those held by the passive party $A$ and the others held by the active party $B$, respectively. Moreover, the top model (i.e., from the second layer to the final layer) is owned by party $B$.

In each iteration, VFL runs a forward pass to make predictions and a backward pass to update parameters (only a forward pass exists in the inference phase). In the forward pass, party $A$ computes the output of its local model using its own data and sends the results to party $B$. Then party $B$ aggregates the first layer output and runs the top model to obtain the final output. In the backward pass, party $B$ computes the gradients using the labels and updates all the local parameters. Party $A$ receives intermediate gradients from party $B$ and computes the local model gradients.

Formally, let $d_A$ and $d_B$ denote the number of input features of party $A$ and $B$, respectively. Consider a neural network with a weight matrix $\W\in\mathbb{R}^{k\times (d_A + d_B)}$ in the input layer. Here $d_A + d_B$ is the total input dimension and $k$ is the number of neurons in the second layer. In each iteration, party $A$'s input is a vector $\x_A\in\mathbb{R}^{d_A}$ and party $B$'s input is a vector $\x_B\in\mathbb{R}^{d_B}$. The weight $\W$ is vertically partitioned into two matrices $\W_A\in\mathbb{R}^{k\times d_A}$ and $\W_B\in\mathbb{R}^{k\times d_B}$, such that party $A$ owns $\W_A$ while party $B$ owns $\W_B$. The remaining parameters, denoted by $\bm{\theta}$, are owned by party $B$.

In each iteration, party $A$ sends an intermediate result $\z_A = \W_A\x_A$ to party $B$. Then party $B$ computes $\z_B = \W_B\x_B$ and $\z = \z_A + \z_B$. After obtaining $\z$, the aggregated output of the first layer, party $B$ completes the forward pass by computing $f_{\bm{\theta}}(\z)$. Here $f_{\bm{\theta}}$ denotes the remaining forward computation, which is done only by party $B$.

In the backward pass, party $B$ uses the label $y$ to compute the gradients of loss $L$ w.r.t. $\bm{\theta}$ and $\z$. The gradient $\frac{\partial L}{\partial \W_B}$ is obtained by $\frac{\partial L}{\partial \z}\x_B^\top$. Thus all parameters maintained by party $B$ can be updated by the gradient descent method. To update $\W_A$, party $B$ passes $\frac{\partial L}{\partial \z}$ to party $A$. Party $A$ can then calculate $\frac{\partial L}{\partial \W_A} = \frac{\partial L}{\partial \z}\x_A^\top$.

The fundamental advantage of cutting at the first layer is that this causes no changes in the overall model structure, while cutting at other layers can damage some of the internal connections as shown in Figure~\ref{fig:cut}. Hence, existing model architectures can be readily applied. To see this, let $\x\in\mathbb{R}^{d_A + d_B}$ be the concatenation of $\x_A$ and $\x_B$. One can observe that $\W_A\x_A + \W_B\x_B = \W\x$, \peng{which indicates the aggregated value of the outputs of bottom models equals the first layer output in the centralized case.}

\begin{figure}
 \centering
 \includegraphics[width=\columnwidth]{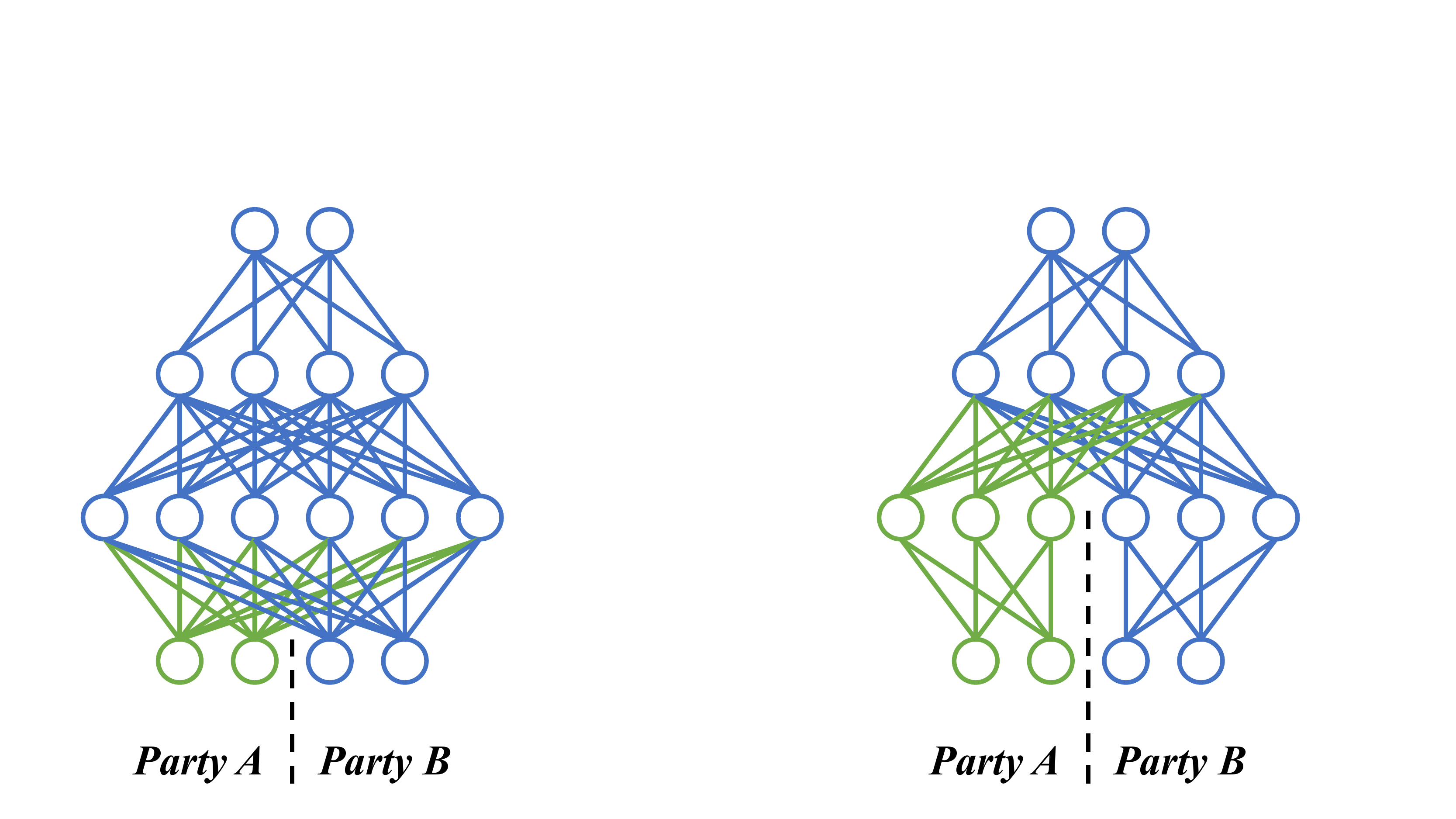}\\
 \caption{NN structure when cut layer is the input layer (left) and when cut layer is the second layer (right).}
 \label{fig:cut}
\end{figure}

\subsection{Threat Model}
Our threat model assumes that the active party is a semi-honest adversary, i.e., the adversary will strictly follow the training procedure, but tries to extract private information from its view. Specifically, the view of the adversary (i.e., the active party) includes its input data, local model parameters, and intermediate results received during the training process. However, it knows nothing about the model weights of the passive party. In this work, we consider that the goal of the adversary is to perform a data reconstruction attack in that the active party tries to reconstruct the passive party's input data.

We now formally describe the data reconstruction attack. Suppose training runs for $T$ rounds. Let $\{\x_A^t\}_{t = 1, \dots, T}$ and $\{\x_B^t\}_{t = 1, \dots, T}$ denote the input features of party $A$ and $B$ in each round. Party $B$ holds the label $\{y^t\}_{t = 1, \dots, T}$. Let $\{\W_B^t\}_{t = 1, \dots, T}$ and $\{\bm{\theta}^t\}_{t = 1, \dots, T}$ denote the parameters of party $B$'s bottom and top model in each training round. $\{\z_A^t\}_{t = 1,\dots, T}$ is the intermediate results received by party $B$. The goal of party $B$ is to reconstruct the input features of party $A$, i.e., to find an algorithm $\mathcal{A}$, so that
\[\mathcal{A}(\{\x_B^t, y^t, \W_B^t, \bm{\theta}^t, \z_A^t\}_{t = 1,\dots, T}) = \{\x_A^t\}_{t = 1, \dots, T}.\]

Extracting all the input features can be particularly challenging. In practice, it is considered to be a big threat even if one of the features could be reconstructed. For example, inferring the gender of customers definitely causes a privacy breach. Thus, it is natural to consider a relaxation of data reconstruction attack, which aims at reconstructing one of the features, i.e., find an algorithm $\mathcal{A}$, so that
\[\mathcal{A}(\{\x_B^t, y^t, \W_B^t, \bm{\theta}^t, \z_A^t\}_{t = 1,\dots, T}) = \{(\x_A^t)_i\}_{t = 1, \dots, T}\]
for some $i\in [d_A]$.

\subsection{Privacy Leakage}
Noticing what party $B$ receives is a matrix product $\z_A = \W_A\x_A$. Although no private data $\x_A$ is transmitted, it is natural to ask if party $B$ could infer $\x_A$ from the product $\z_A$. This cannot be done simply by solving the linear equations, because both $\W_A$ and $\x_A$ are unknown to party $B$. Indeed, since $\z_A$ is a matrix product of $\W_A$ and $\x_A$, there are an infinite number of possible pairs $\W_A$ and $\x_A$ that generate the same $\z_A$. Thus, party $B$'s view can have infinite number of valid inputs of party $A$, which makes it theoretically impossible to reconstruct party $A$'s input. We now state this impossibility result formally in the following theorem:
\begin{theorem}
\label{thm:impossibility}
Suppose $\{\z_A^t\}_{t = 1,\dots, T}$ is the set of intermediate results received by party $B$ during training (inference). There are infinite possible pairs of initial weight $\W_A^0$ and input data $\{\x_A^t\}_{t = 1, \dots, T}$ that can generate this set. Thus, party $B$ cannot reconstruct party $A$'s input.
\end{theorem}
\begin{proof}
In $t$-th iteration ($1\le t\le T$), party $A$ sends $\z_A^t = \W_A^{t - 1} \x_A^t$ to party $B$. Then it receives gradient $\g^t$ (w.r.t. $\z_A^t$) from party $B$ and update the weight by $\W_A^t = \W_A^{t - 1} - \eta_t\g^t(\x_A^t)^\top$, where $\eta_t$ is the learning rate in $t$-th iteration (for inference phase, just set the learning rate to be $0$).

Suppose $\W_A^0$ and $\{\x_A^t\}_{t = 1, \dots, T}$ is a pair of initial weight and input data that generates set $\{\z_A^t\}_{t = 1,\dots, T}$. Let $\U\in\mathbb{R}^{d_A \times d_A}$ be an arbitrary unitary matrix. We prove that adopting $\W_A^0\U^\top$ and $\{\U\x_A^t\}_{t = 1, \dots, T}$ as initial weight and input data leads to the same set of intermediate output $\{\z_A^t\}_{t = 1,\dots, T}$.

Consider the first iteration, the party $A$ first computes $\W_A^0\U^\top \U\x_A^1$, which is exactly the same as $\z_A^0 = \W_A^0 \x_A^1$. Since it sends the same intermediate output to party $B$. It receives the same gradient $\g^1$. It then computes $\W_A^0\U^\top - \eta_t\g^t(\U\x_A^t)^\top = (\W_A^0 - \eta_t\g^1(\x_A^1)^\top)\U^\top = \W_A^1\U^\top$ to update its local weight.

We can then prove by induction that in $t$-th iteration, the intermediate output sent to party $B$ is exactly $\z_A^t$ and the local weight held by party $A$ is $\W_A^t\U^\top$. We have shown that this holds for $t = 1$.

Suppose this holds for iteration $1$ to $t - 1$. In $t$-th iteration party $A$ sends $\W_A^{t - 1}\U^\top \U\x_A^t = \W_A^{t - 1} \x_A^t = \z_A^t$ to party $B$. Note that party $B$ receives $\{\z_A^1, \dots, \z_A^t\}$ until iteration $t$. The gradient it sends back to party $A$ must be $\g^t$. Therefore the weight held by party $A$ will be updated to $\W_A^{t - 1}\U^\top - \eta_t\g^t(\U\x_A^t)^\top = (\W_A^{t - 1} - \eta_t\g^t(\x_A^t)^\top)\U^\top = \W_A^t\U^\top$.

Thus $\W_A^0\U^\top$ and $\{\U\x_A^t\}_{t = 1, \dots, T}$ generate the same set $\{\z_A^t\}_{t = 1,\dots, T}$. As there are infinite unitary matrices of size $d_A\times d_A$, also infinite pairs of initial weight and input data.

Since $\{\z_A^t\}_{t = 1,\dots, T}$ is the only information that party $B$ receives, an attack algorithm will always output the same for these pairs. However, $\{\U\x_A^t\}_{t = 1, \dots, T}$ varies for different $\U$. Thus, such an attack algorithm doesn't exist.
\end{proof}

\begin{remark}
In the above proof, we allow the intermediate gradient $\g^t$ to be generated arbitrarily. That is, party $B$ doesn't have to follow protocol. This is called the malicious adversary setting. A malicious adversary is more powerful than a semi-honest one at attacking. Therefore, we actually prove a stronger result - even a malicious adversary cannot reconstruct the passive party's input.
\end{remark}

\begin{remark}
We use vanilla stochastic gradient descent (SGD) in the proof of Theorem~\ref{thm:impossibility}. It can be directly extended to other popular variants such as SGD with momentum~\cite{rumelhart1986learning}, RMSprop~\cite{tieleman2012lecture}, and Adam~\cite{kingma2014adam} because the update only depends on historical gradients. This theorem indicates that one cannot distinguish between infinite possible inputs. Thus there is no way to recover the data.
\end{remark}

This illustrates that an attack is not possible when the attacker has zero knowledge about the data. In practice, however, the active party may know certain properties of the passive party's input. Noticing that the impossibility result relies on the fact that performing a unitary transform on the input doesn't change the active party's view. Thus, an intermediate result corresponds to an infinite number of possible inputs, which are indistinguishable to an attacker. However, when an attacker knows certain properties of the input features, the number of possible inputs could be drastically reduced (even to only one), making it possible for the attacker to perform attacks.

\subsection{Binary Assumption}
In this paper, we are particularly interested in the scenario where the passive party has some binary input features, i.e., features that take values only $0$ or $1$. This kind of feature is very common in real-world situations, such as clients' marital status (married and unmarried), exam results (pass or fail), and outcomes of medical tests (positive or negative). When they are converted into numerical values, the features will contain values $0$ and $1$ (we will show that our attack also works for other values as long as the number is $2$ in Section~\ref{sec:discuss}). When there are binary attributes in the raw data, it is inevitable that they will be converted into input features with two values.

The binary features may also come from feature engineering. In practice, it is common to convert a categorical feature to a one-hot representation, which introduces a large number of binary features. One-hot encoding is frequently used when the raw feature contains many categories but is nominal, i.e., there is no quantitative relationship between different values. For instance, blood types have four categories. Simply assigning them with different numbers implicitly introduces an order between them, which may hinder the model from learning the true relationship. One could apply one-hot encoding to creating dummy variables for each type to achieve better model performance.

The active party's goal is to recover those binary features, i.e., find an algorithm $\mathcal{A}$ so that
\[\mathcal{A}(\{\x_B^t, y^t, \W_B^t, \bm{\theta}^t, \z_A^t\}_{t = 1,\dots, T}) = \{(\x_A^t)_i\}_{t = 1, \dots, T},\]
where feature $i$ is binary. \peng{Although in Theorem~\ref{thm:impossibility} we demonstrate that there can be an infinite number of possible inputs that generate the same intermediate results. If we restrict the features to be binary, the number of possible inputs become quite limited, making it possible to reconstruct the data.} 

\section{Binary Feature Inference Attacks}\label{sec:attack}
In this section, we first prove that a binary feature inference attack is feasible under our assumptions, albeit with provable NP-hardness (Sec.~\ref{sec:attack_feasibility}). We then present two attack algorithms with empirical runtime results, which concretely demonstrate how to perform the attack in practice.

\subsection{Feasibility}~\label{sec:attack_feasibility}
\PHB{Solvability.} Our attack is based on a collection of intermediate results from the weight matrix $\W_A$. Note that the weight matrix remains the same in the inference phase. Formally, we consider that we have obtained $n$ intermediate results represented by a matrix $\Z_A = \X_A\W_A^\top\in\mathbb{R}^{n\times k}$, where $\X_A$ is an $n\times d_A$ matrix with each row containing one input data record. Here we transpose a data record from a column vector to a row vector for convenience. Thus, the weight matrix $\W_A$ is also transposed in the formula. We can suppose $n > d_A$ because $n$ can be arbitrarily large during inference, or one may set a batch size larger than $d_A$.

There are three key observations that are useful to our attack: (1) The number of neurons in the second layer (i.e. the first hidden layer) is larger than the number of neurons in the input layer. The hidden units play the role of capturing nonlinearities in the input data. With more neurons, the model could produce better predictions. Training a reasonably large number of neurons with regularization is common in practice~\cite{hastie2009elements}. In VFL, participants are usually big companies with computation capability, prompting them to put down more neurons for better performance. (2) The weight matrix $\W_A$ has full rank. Intuitively, the matrix is randomly initialized and it is updated by a random set of points with a specific learning rate in each iteration. It appears to be unlikely that this matrix will become singular at a time. In~\cite{martin2021implicit}, Martin \emph{et al.} empirically show that weight matrices of well-trained DNNs won't undergo rank collapse, i.e., they retain full rank during training. (3) Linear combinations of input features (i.e., columns of $\X_A$) are not binary except the input binary features. For non-binary features (i.e., features with more than two possible values), it is unlikely that their linear combinations take only two values. Even for binary features, their linear combination often contains values other than $0$ or $1$, unless the features are strongly correlated (e.g., features created by one-hot encoding, or one attribute implies another), which can still be considered as privacy leakage if it is uncovered. In Section~\ref{sec:eval}, our experimental results show that we only recover features that exactly match the inputs, further validating this observation.

The first observation indicates $k \ge d_A + d_B > d_A$. So, we have $rank(\W_A) = d_A$ by the fact that it has full rank. Thus, the intermediate outputs $\Z_A$ share the same column span with the input matrix $\X_A$. To recover an input binary feature, one could try to find a binary vector in the column span of $\Z_A$. From the third observation, it is very likely to be one actual input binary feature.

By the fact that $rank(\Z_A) \le rank(\W_A) = d_A$, we only need to consider an $n\times d$ matrix $\A$ which shares the same column space of $\Z_A$ ($\A$ can be obtained, for example, by picking $d$ linearly independent columns of $\Z_A$), where $d = rank(\Z_A)$. Finding a binary vector in the column span of $\Z_A$ is equivalent to finding it in the column span of $\A$. For simplicity, we consider the attack problem is to find a binary vector $\x$ in the column space of $\A$. Strictly speaking, we want to find a binary vector $\x$ so that there exists vector $\w\in\mathbb{R}^{d}$ that $\A\w = \x$.

\PHM{NP-hardness.} Under the binary assumption, we can formulate the attack as finding a binary vector in the column span of a given matrix based on the three observations above. However, it is still challenging. 
\peng{As we will show that the decision version of the attack problem is NP-hard by using a reduction from the Exact Cover problem, which is known to be NP-hard. We formalize the reduction in the following Theorem~\ref{thm:nphardness}.}

\begin{theorem}
\label{thm:nphardness}
Given a matrix $\A$, deciding whether there is a nonzero binary vector $\x$ so that there exists $\w$ that $\A\w = \x$ is NP-hard.
\end{theorem}

\begin{proof}
 We start with a quick primer on the exact cover problem. Given a set of $n$ elements $U = \{u_1, \dots, u_n\}$ and a collection $C = \{S_1, \dots, S_m\}$ of subsets of $U$. The \textit{exact cover problem} is to decide whether there is a sub-collection $C'\subseteq C$ that covers every element exactly once, i.e. $|\{j|u_i\in S_j\text{ and }S_j\in C'\}| = 1$ for all $i\in [n]$. Given this NP-hard problem, we are now ready to prove that the decision version of the attack problem is also NP-hard.

 Consider an instance of Exact Cover: $U = \{u_1, \dots, u_n\}$ and $C = \{S_1, \dots, S_m\}$.

 We construct three matrices. Let $\A_1 = \I_{m + 1}$ be an $(m +1)$-dimensional identity matrix, $\A_2 = (a_{ij})\in\mathbb{R}^{n\times (m + 1)}$ where
 \[a_{ij} = \begin{cases}
 1 & \text{if }j\in [m]\text{ and }u_i\in S_j \\
 0 & \text{if }j\in [m]\text{ and }u_i\notin S_j \\
 -1 & \text{if }j = m + 1
 \end{cases},\]
 $\A_3 = [2|S_1|, \dots, 2|S_m|, -2n]\in\mathbb{R}^{1\times (m + 1)}$. Stacking the three matrices we can obtain
 \[\A = \begin{bmatrix}
 \A_1 \\
 \A_2 \\
 \A_3
 \end{bmatrix} \in \mathbb{R}^{(n + m + 2)\times (m + 1)}.\]

 We then show if we could decide whether there exists $\w\in\mathbb{R}^{m + 1}$ that $\x = \A\w$ is nonzero and binary, we can decide whether there exists an exact cover.

 Suppose there is a sub-collection $C'$ that covers every element exactly once. Let $\w_{m + 1} = 1$ and for $j\in[m]$ let
 \[\w_j = \begin{cases}
 1 & \text{if } S_j\in C' \\
 0 & \text{if } S_j\notin C'
 \end{cases}.\]
 Then $\A_1 \w =\w$ is nonzero and binary. The $i$-th element of $\A_2\w$ is $\sum_{j=1}^m a_{ij}\w_j - 1 = |\{j | u_i\in S_j \text{ and }S_j\in C'\}| - 1 = 0$. And $\A_3\w = 2 (\sum_{S_j\in C'}|S_j| - n) = 0$. Therefore $\A\w$ is nonzero and binary.

 Now, suppose there exists $\w\in\mathbb{R}^{m + 1}$ that $\x = \A\w$ is nonzero and binary. From $\A_1\w = \w$, we know that $\w$ is also binary. Since $\A_3 \w = 2\sum_{j = 1}^{m}|S_j|\w_j - 2n\w_{m + 1}$ is a multiple of $2$, it must be $0$. Thus $\w_{m + 1}$ must be $1$ otherwise we will have $\w = \bm{0}$, contradicting that $\A\w$ is nonzero. So if we let $C' = \{S_j | \w_j = 1, j\in[m]\}$, we have $\sum_{S_j \in C'} |S_j| = \sum_{j = 1}^{m}|S_j|\w_j = n\w_{m + 1} = n$. The $i$-th element of $\A_2\w$ is $\sum_{j=1}^m a_{ij}\w_j - 1 = |\{j | u_i\in S_j \text{ and }S_j\in C'\}| - 1$, which is either $0$ or $1$, indicating that $|\{j | u_i\in S_j \text{ and }S_j\in C'\}|$ should be $1$ or $2$ for each $i\in[n]$. But we have $\sum_{i=1}^{n}|\{j | u_i\in S_j \text{ and }S_j\in C'\}| = \sum_{S_j\in C'}|S_j| = n$. Thus $|\{j | u_i\in S_j \text{ and }S_j\in C'\}|$ must be 1, namely, each element is covered exactly once. Therefore $C'$ is an exact cover.
\end{proof}

\begin{remark}
 In Theorem~\ref{thm:nphardness} we add one restriction that $\x$ should be a nonzero vector, i.e. $\x\neq\bm{0}$. Since $\A\bm{0} = \bm{0}$, the zero vector $\bm{0}$ is always a trivial solution. Thus, it reveals no information about the actual input. Also, an input feature that contains only $0$ contributes nothing to the training process and is impossible to be detected by the attacker.
\end{remark}

\PHM{Summary.} While it is computationally prohibitive for an attacker to recover the passive party's binary feature vector when the feature dimension grows exponentially large, In practice, however, an exponential-time complexity algorithm can still be effective to perform the attack within a reasonable time. Moreover, the attack is conducted offline as the active party only needs to collect the intermediate results once. We next present two attack algorithms.

\subsection{Strawman Attack: Solving by Linear Equations}~\label{sec:attack_strawman}
\PHB{Technical intuition.}
Recall that the problem now is to find a binary vector $\x$ so that the linear equations $\A\w = \x$ have a solution. One direct approach is to try all possible $\x$'s and check if such $\w$ exists. This requires solving the linear equations $\Theta(2^n)$ times, which is not acceptable since $n$ can be very large. Noticing that the rank of $\A$ is only $d$, thus $d$ linear equations are sufficient to derive a solution. Therefore, we can solve the problem on a $d\times d$ submatrix of $\A$ and check the solutions on the original matrix $\A$, reducing the number of enumerations from $\Theta(2^n)$ to $\Theta(2^d)$.

\begin{algorithm}
\caption{Attack by Solving Linear Equations}
\label{alg:attack_equ}
\begin{algorithmic}[1]
 \REQUIRE Matrix $\A\in\mathbb{R}^{n\times d}$ with $rank(\A) = d$
 \STATE Find a $d\times d$ submatrix $\A'$ from $\A$, where $rank(\A') = d$
 \FOR{$\x'$ in $\{0,1\}^{d} \setminus \{\bm{0}\}$}
 \STATE $\w \gets \A'^{-1}\x'$ \label{alg:attack_equ_1}
 \STATE $\x\gets \A\w$ \label{alg:attack_equ_2}
 \IF{$\x$ is binary}
 \RETURN $\x$
 \ENDIF
 \ENDFOR
 \ENSURE Vector $\x\in\mathbb{R}^n$
\end{algorithmic}
\end{algorithm}

\PHB{Correctness.}
To verify the correctness, we have to ensure that if $\A$ does contain some binary vectors in its column space, Algorithm~\ref{alg:attack_equ} will always find one of them. This is implied by the fact that picking $d$ linearly independent rows from $\A\w = \x$ preserves the uniqueness of $\w$ and the right-hand side remains binary. Therefore $\w$ will be obtained by solving the reduced problem and $\x$ will be output when we return to the original one. We formally state the correctness guarantee in Theorem~\ref{thm:attack_equ} and its proof.

\begin{theorem}
Suppose $\A$ is an $n\times d$ matrix with rank $d$. If there exists a vector $\w'$ such that $\A\w'$ is nonzero and binary, then Algorithm~\ref{alg:attack_equ} outputs a nonzero and binary vector $\x=\A\w$ for some $\w$. Moreover, Algorithm~\ref{alg:attack_equ} runs in time $O(nd\cdot 2^{d})$.
\label{thm:attack_equ}
\end{theorem}

\begin{proof}
If a vector is returned by the algorithm, it must be a valid solution. thus, we only need to show that it always output some vector(s). Since $rank(\A) = d$, Algorithm~\ref{alg:attack_equ} can find a $d\times d$ submatrix $\A'$ with full rank. Clearly, $\A'\w'$ is binary because $\A\w'$ is binary. Note that $\w'$ is nonzero because $\A\w'$ is nonzero, thus $\A'\w'$ is also nonzero. Otherwise, we will have $rank(\A') < d$, contradicting that $\A'$ has full rank. When the algorithm processes with $\x' = \A'\w'$, $\x = \A\w'$ will be generated and output.

The algorithm enumerates all $2^{d}$ possible $d$-dimensional binary vectors. For each vector it computes two matrix multiplications, which cost $O(nd)$ time. Thus the total time complexity is $O(nd\cdot 2^{d})$.
\end{proof}

\PHM{Time complexity.} The computation complexity of Algorithm~\ref{alg:attack_equ} is $O(nd\cdot 2^{d})$\peng{, which becomes $O(nd_A\cdot 2^{d_A})$ when performing on the intermediate results $\Z_A$}. To empirically examine the feasibility of the attack, we conducted a testbed evaluation. Specifically, we implement it using C++, one of the most efficient languages with Eigen library\cite{eigenweb} for matrix operations. We then deploy it atop an Amazon EC2 c5.4xlarge instance (16 GiB Memory and 32 vCPU), a middle-tier compute-optimized commercial virtual machine.

\begin{figure}[t]
 \centering
 \begin{subfigure}[b]{0.45\columnwidth}
 \centering
 \includegraphics[width=\columnwidth]{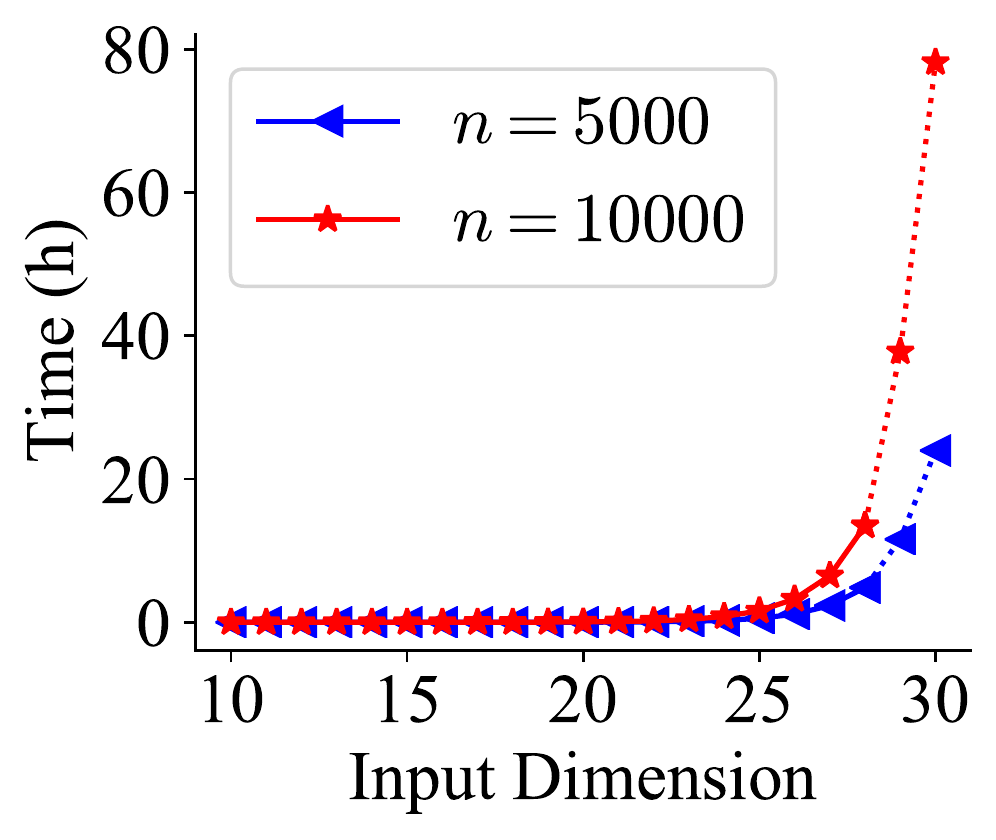}
 \caption{Algorithm 1}
 \label{fig:runtime_equ}
 \end{subfigure}
 \begin{subfigure}[b]{0.45\columnwidth}
 \centering
 \includegraphics[width=\columnwidth]{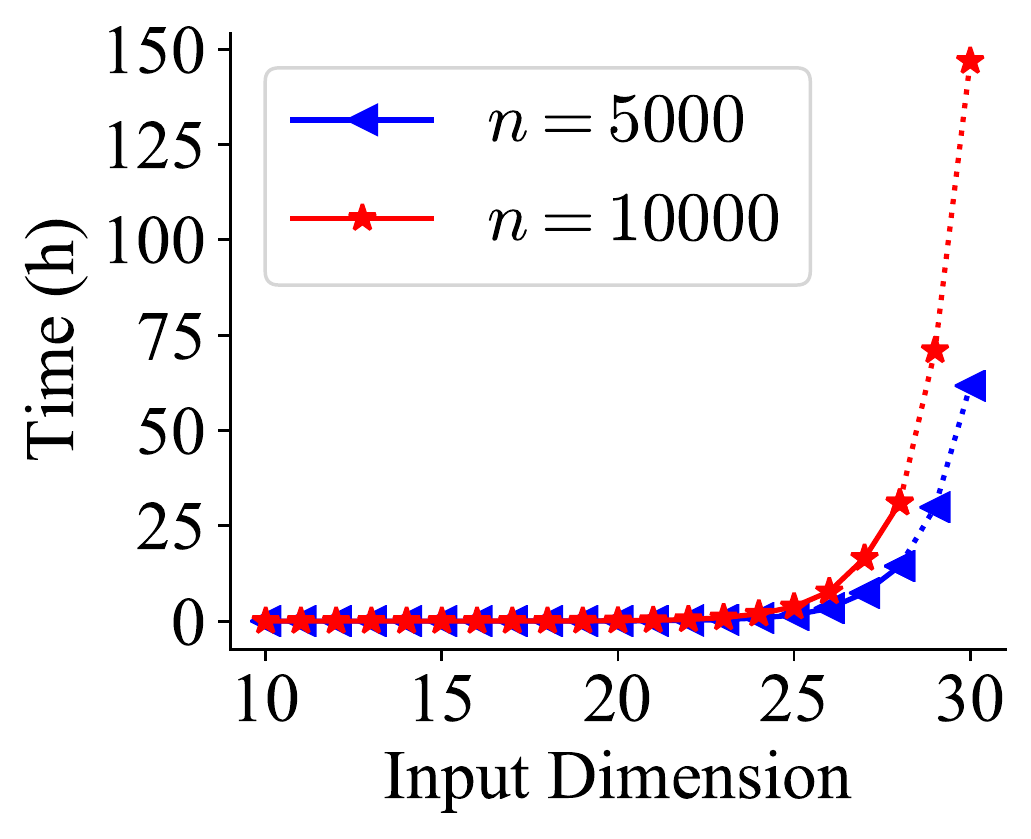}
 \caption{Algorithm 2}
 \label{fig:runtime_reg}
 \end{subfigure}
 \caption{Attack runtime for different $d_A$.}
 \label{fig:runtime}
 \end{figure}

Since the runtime of Algorithm~\ref{alg:attack_equ} only depends on the input size (i.e. number of data records $n$ and input dimension of passive party $d_A$), we run it on randomly generated input data to test its efficiency. In our experiments we try $n\in\{5000, 10000\}$. For $d_A$ from $10$ to $28$ we report the real runtime and for $29$ and $30$ the data is approximated by previous points using function $f(d_A) = C\cdot 2^{d_A}\cdot d_A$. We plot the results in Figure~\ref{fig:runtime_equ}.

From the results, we can see the runtime grows nearly linearly with $n$ and exponentially with $d_A$. In practice, this attack is performed offline as the attacker can record the intermediate output. Thus, there is no strict constraint on time - it can run this attack algorithm for weeks or months to obtain the results. Also, in real-world applications in VFL, the attacker is a company with ample computation capability. The attack algorithm can be easily parallelized and run much faster on GPUs, thus, the training framework can still be vulnerable even if the input dimension is much larger than that used in our experiments.

\subsection{Attack: Solving Linear Regression}\label{sec:attack_robust}
\PHB{Technical intuition.} Algorithm~\ref{alg:attack_equ} performs the attack via solving linear equations, which may not be feasible in practice because the equality constraints can be easily broken by numerical errors during calculations, or random noise proactively added by the passive party.

To sidestep the drawbacks of solving linear equations, instead of finding a binary vector in the column span of matrix $\A$, we can target to find a binary vector that is close to the column space. To be exact, instead of solving \[\{\x\in\{0, 1\}^n\setminus\{\bm{0}\}: \A\w = \x \text{ for some }\w\}\] by Algorithm~\ref{alg:attack_equ}, one can obtain $\x$ by computing

\begin{equation*}
 \argmin_{\x\in\{0, 1\}^n\setminus\{\bm{0}\}}\min_{\w}\Vert \A\w - \x\Vert_2^2,
\end{equation*} which minimizes the Euclidean distance from the column space of $\A$ to $\x$.

However, solving the above problem by enumeration requires solving the linear regression problem in $\Omega(2^n)$ times, where the runtime could be prohibitively high. We use a similar idea as in Algorithm~\ref{alg:attack_equ}: solve the problem on a submatrix, then verify the solutions on the whole matrix. The challenge is how to select a submatrix, whose solutions deviate from the original matrix with a bounded error. We adopt the Leverage Score Sampling technique~\cite{Mahoney2011Randomized} to sample the submatrix. The details are presented in Algorithm~\ref{alg:attack_reg}.

\begin{algorithm}
\caption{Attack by solving Linear Regression}
\label{alg:attack_reg}
\begin{algorithmic}[1]
 \REQUIRE Matrix $\A\in\mathbb{R}^{n\times d}$ with $rank(\A) = d$
 \STATE Use leverage score sampling to randomly sample and rescale $r$ rows and obtain $\A' = \D\S\A$, where $\S\in\mathbb{R}^{r\times n}$ is a sampling matrix that samples $r$ rows $R = \{i_1, \dots, i_r\}$ and $\D\in\mathbb{R}^{r\times r}$ is a diagonal matrix that rescales the values in each row
 \STATE $T\gets \{\e\}$ where $\e = [1, 0, \dots, 0]^\top\in\mathbb{R}^n$
 \FOR{$\x'$ in $\{0,1\}^r \setminus \{\bm{0}\}$}
 \STATE $\w' \gets \argmin_{\w}\Vert \A'\w - \D\x'\Vert_2^2$
 \STATE Create an $n$-dimensional vector $\x$, set \\
 $\x_i \gets \begin{cases}\x'_j & \text{if }i=i_j \text{for some }j\\ 0 & \text{if } i \notin R\text{ and }(\A\w')_i < 0.5\\ 1 & \text{otherwise}\end{cases}$
 \STATE $T\gets T\cup\{\x\}$
 \ENDFOR
 \STATE $\x^* \gets \argmin_{\x\in T}\min_{\w}\Vert \A\w - \x\Vert_2^2$
 \ENSURE Vector $\x^*\in\mathbb{R}^n$
\end{algorithmic}
\end{algorithm}

\PHM{Correctness.} The correctness of the approximation algorithm is based on the following property of the Leverage Score Sampling technique.
\begin{lemma}[Leverage Score Sampling~\cite{Mahoney2011Randomized}]
\label{lem:leverage}
Given an $n\times d$ matrix $\A$, an $n$-dimensional vector $\b$, and $\epsilon > 0$. Let $\p_i = \Vert \U_{(i)}\Vert_2^2 / d$ be normalized leverage scores, where $\U$ is the matrix containing left singular vectors of $\A$ and $\U_{(i)}$ is the $i$-th row of $\U$. Let $\S\in\mathbb{R}^{r\times n}$ and $\D\in\mathbb{R}^{r\times r}$ be the sampling and rescaling matrix generated from distribution $\p$, where $r = O(d\log d/\epsilon^2)$. With constant probability we have
\[\Vert \A\Tilde{\x} - \b\Vert_2^2\le (1+\epsilon)\min_{\x}\Vert \A\x - \b\Vert_2^2,\]
where $\Tilde{\x} = \argmin_{\x}\Vert \D\S\A\x - \D\S\b\Vert_2^2$.
\end{lemma}

It states that, one could sample a submatrix with an appropriate size, so that the solution obtained from solving the linear regression on this submatrix produces only a multiplicative error of $\epsilon$ on the original matrix. With that, the relative performance guarantee of Algorithm~\ref{alg:attack_reg}, which states that it can return a solution with bounded error $\epsilon$, can be ensured.

\begin{theorem}
Given an $n\times d$ matrix $\A$ and $\epsilon > 0$, let $\x^{opt} = \argmin_{\x\in \{0, 1\}^n\setminus\{\bm{0}\}} \min_{\w}\Vert \A\w - \x\Vert_2^2$. By choosing $r = O(d\log d /\epsilon^2)$, Algorithm~\ref{alg:attack_reg} outputs a vector $\x^*$ such that
\[\min_{\w}\Vert \A\w - \x^*\Vert_2^2 \le (1+\epsilon)\min_{\w}\Vert \A\w - \x^{opt}\Vert_2^2\]
with constant probability. Moreover, Algorithm~\ref{alg:attack_reg} runs in time $O(nr\cdot 2^{r})$.
\end{theorem}

\begin{proof}
By Lemma~\ref{lem:leverage}, we have
\[\Vert \A\w' - \x^{opt}\Vert_2^2 \le (1+\epsilon)\min_{\w}\Vert \A\w - \x^{opt}\Vert_2^2,\]
where $\w' = \argmin_{\w}\Vert \D\S\A\w - \D\S\x^{opt}\Vert_2^2$, with constant probability.

If $\S\x^{opt}$ is nonzero, consider the vector $\x$ generated in the algorithm with $\x' = \S\x^{opt}$, it is easy to see that $\S\x = \x' = \S\x^{opt}$. Thus $\w' = \argmin_{\w}\Vert \D\S\A\w - \D\S\x\Vert_2^2$. So we have
\begin{align*}
 \Vert \A\w' - \x\Vert_2^2 &= \sum_{i\in R} [(\A\w')_i - \x_i]^2 + \sum_{i\notin R} [(\A\w')_i - \x_i]^2 \\
 &\le \sum_{i\in R} [(\A\w')_i - \x^{opt}_i]^2 + \sum_{i\notin R} [(\A\w')_i - \x^{opt}_i]^2 \\
 &= \Vert \A\w' - \x^{opt}\Vert_2^2
\end{align*}
because $[(\A\w')_i - \x_i]^2 = \min([(\A\w')_i - 0]^2, [(\A\w')_i - 1]^2)\le [(\A\w')_i - \x^{opt}_i]^2 $ for $i\notin R$.

Since $\x\in T$, by the definition of $\x^*$, we have
\begin{align*}
 \min_{\w}\Vert \A\w - \x^*\Vert_2^2 &\le \min_{\w}\Vert \A\w - \x\Vert_2^2 \\
 &\le \Vert \A\w' - \x\Vert_2^2 \\
 &\le \Vert \A\w' - \x^{opt}\Vert_2^2 \\
 &\le (1+\epsilon)\min_{\w}\Vert \A\w - \x^{opt}\Vert_2^2.
\end{align*}

If $\S\x^{opt}$ is a zero vector, $\argmin_{\w} \Vert \D\S\A\w - \D\S\x^{opt}\Vert_2^2$ is also a zero vector. Thus we have
\[\Vert \x^{opt}\Vert_2^2 \le (1 + \epsilon)\min_{\w}\Vert \A\w - \x^{opt}\Vert_2^2.\]

Since $\e\in T$ and $\x^{opt}$ is nonzero, we have
\begin{align*}
 \min_{\w}\Vert \A\w - \x^*\Vert_2^2 &\le \min_{\w}\Vert \A\w - \e\Vert_2^2 \\
 &\le \Vert \e \Vert_2^2 \\
 &\le \Vert \x^{opt} \Vert_2^2 \\
 &\le (1+\epsilon)\min_{\w}\Vert \A\w - \x^{opt}\Vert_2^2.
\end{align*}

The algorithm enumerates all binary vectors in $\mathbb{R}^{r}$ and for each vector it solves a least square problem, which can be done in $O(nr)$ time because the pseudo inverse of $\A'$ can be precomputed. Thus the total time complexity is $O(nr\cdot 2^r)$.
\end{proof}

\PHM{Time complexity.} 
The computation complexity of Algorithm~\ref{alg:attack_reg} is $O(nr\cdot 2^r)$, which depends on the choice of $r$. Our evaluation in Section~\ref{sec:eval} shows that it is sufficient to choose $r = d_A + 1$ instead of matching the theoretical bound. We follow the experimental setup of Algorithm~\ref{alg:attack_equ}, and run Algorithm~\ref{alg:attack_reg} with $r=d_A+1$.

The results are plotted in Figure~\ref{fig:runtime_reg}, where we observe a similar performance trend with Algorithm~\ref{alg:attack_equ}, though the constant is slightly inflated approximately by a factor of 2. Note that Algorithm~\ref{alg:attack_reg} is more robust than Algorithm~\ref{alg:attack_equ} in the presence of numerical perturbation, i.e., it trades time for robustness, which is beneficial for privacy preservation (as we will discuss in Sec.~\ref{sec:eval}). Moreover, Algorithm~\ref{alg:attack_reg} is also designed to run offline, similar to Algorithm~\ref{alg:attack_equ}.

\PHM{Summary.} We have devised two attack algorithms to perform binary feature inference attacks in VFL, where one has lower time complexity and the other is more robust in the presence of perturbation. Although both algorithms have an exponential time complexity, the empirical results indicate that the runtime is quite acceptable in practice unless the number of features in the passive party's data is exceedingly large.

\section{Defense}\label{sec:defense}
Our devised attacks naturally raise the question of whether the passive party's binary features can be protected. In this section, we answer this question affirmatively by presenting two effective defense mechanisms.

\subsection{Technical Intuitions}\label{sec:defense_intuition}
\PHB{Technical intuition.} Intuitively, to increase the difficulty of extracting useful information from the passive party's intermediate results, a naive yet effective approach is to (slightly) perturb the results with random noise. By adding noise, the linear equations our attack relies on no longer hold, thereby rendering the attack ineffective.

\PHM{Gaussian Noise Masking method.} Specifically, in the \textit{forward} pass, after computing the intermediate output $\z_A = \W_A\x_A$ based on its own weight, the passive party generates a random mask $\u$ that has the same size as $\z_A$. Each element of $\u$ is i.i.d. drawn from a Gaussian distribution $\mathcal{N}(0, \sigma^2)$ with zero-mean and variance $\sigma^2$.

The passive party then passes $\z_A + \u$ to the active party in the forward pass. In the respective \textit{backward} pass, the update procedure is the same as if there is no noise added to $\z_A$. That is, upon receiving $\frac{\partial L}{\partial \z}$ from the active party, the passive party will compute $\frac{\partial L}{\partial \W_A} = \frac{\partial L}{\partial \z}\x_A^\top$ and update the weight. Leveraging the idea from \textit{Gaussian Noise Masking} we formalize this in Algorithm~\ref{alg:defend_gauss}.

\begin{algorithm}
\caption{Gaussian Noise Masking Defense}
\label{alg:defend_gauss}
\begin{algorithmic}[1]
 \REQUIRE Input data $\x_A\in\mathbb{R}^{d_A}$, noise parameter $\sigma$
 \STATE $\z_A \gets \W_A\x_A$
 \STATE Sample $\u$ where $\u_i\sim\mathcal{N}(0, \sigma^2)$ for $i\in[k]$
 \STATE Send $\z_A + \u$ to the active party
 \STATE Receive $\frac{\partial L}{\partial \z}$ from the active party \vskip 2pt
 \STATE $\frac{\partial L}{\partial \W_A} \gets \frac{\partial L}{\partial \z}\x_A^\top$ \vskip 2pt
 \STATE Update $\W_A$ by gradient descent
\end{algorithmic}
\end{algorithm}

The Gaussian Noise Masking approach can thwart the strawman version of the binary feature inference attack though it relies on finding an exact solution to some linear equations. While the attack via solving linear regression is less sensitive to random noise, such a defense is still effective if the noise is sufficiently high (Section~\ref{sec:gauatk}).

\PHM{The model accuracy} Gaussian Noise Masking could compromise the training performance. To prevent binary feature inference attacks from utilizing intermediate results generated in any training round, a defense algorithm in place has to protect the entire training process. As such, even if minimum noise is added in each round, the noise accumulated during the entire training process may still be large enough to impact convergence speed or result in poor model accuracy.

\subsection{Masquerade Approach}
\PHB{Technical intuition.} The Gaussian Noise Masking approach described above essentially leads the attacker to find an incorrect solution of the target binary feature vector, which can be arbitrary and depends on the complex interplay between the intermediate results and the sampled noise. To deal with this problem, one possible improvement is to intentionally misguide the attacker to a pre-specified fabricated binary feature, instead of an arbitrary solution for ground truth. In this way, the randomness involved in perturbing the passive party's intermediate results will be significantly reduced, which consequently improves the model accuracy.

\PHM{Basic idea.} A naive approach is to simply add one fabricated binary feature to the inputs. Unfortunately, this does not work since the attacker is able to find out all the input binary features, both the true input binary features and the fabricated one. This indicates that we still need some perturbation to protect the true input features.

We restrict the rank of the weight matrix $\W_A$ to be $d_A-1$ during training. This can be viewed as introducing some perturbation since the input matrix $\X_A$ and the intermediate results $\Z_A$ no longer share the same column span. Then we can insert the fabricated binary feature by adding a mask on $\Z_A$ because $\Z_A$ is not full-rank. As a result, this effectively masquerades as the real input features.

More specifically, we can first approximate the $k\times d_A$ weight matrix $\W_A$ by a rank $d_A - 1$ matrix $\widehat{\W_A}$. Thus the intermediate output $\Z_A = \X_A\widehat{\W_A}^\top\in\mathbb{R}^{n\times k}$ has rank $d_A - 1$. Suppose we are given a binary vector $\v\in\mathbb{R}^{n}$. If we could add a mask on the intermediate result so that $\v$ becomes a vector in the column space of $\Z_A$, the attack algorithm will output $\v$ as a solution.

Towards this end, we add a mask $\R = \v\u^\top$ for some vector $\u\in\mathbb{R}^{k}$ to $\Z_A$. Suppose the attacker pick $d_A$ columns of $\Z_A + \R$ to perform the attack, it will obtain $\A + \R' = \A + \v\u'^\top$, where $\A$ is a $n\times d_A$ matrix with rank $d_A - 1$ and $\u'$ is the vector formed by selecting the corresponding $d_A$ coordinates from $\u$. Since $\A$ has rank $d_A - 1$, there exists vector $\w$ that $\A\w = \bm{0}$. Thus $(\A + \v\u'^\top)\w = \v(\u'^\top\w)$ is in the column span of $\A + \R'$, indicating that the attacker will find $\v$ as the solution by solving linear equations.

\PHM{Integration to the training workflow.}
We now show how to integrate this into the training process.

We explicitly decompose $\widehat{\W_A}$ as the product of two matrices $\P\in\mathbb{R}^{k\times (d_A - 1)}$ and $\Q\in\mathbb{R}^{(d_A - 1) \times d_A}$. Specifically, in the forward pass, the passive party computes
\[ \z_A = \big[\P|\u\big]\begin{bmatrix}\Q\x_A \\ a\end{bmatrix}\]
for an input $\x_A$, where $a$ is randomly set to be $0$ or $1$ with equal probabilities. The passive party then sends $\z_A$ to the active party. In the backward pass, the passive party receives $\frac{\partial L}{\partial \z}$ from the active party. Then it calculates

\begin{equation*}
 \frac{\partial L}{\partial \P} = \frac{\partial L}{\partial \z}\x_A^\top\Q^\top, \frac{\partial L}{\partial \u} = a\frac{\partial L}{\partial \z}, \text{and } \frac{\partial L}{\partial \Q} = \P^\top\frac{\partial L}{\partial \z}\x_A^\top,
\end{equation*} to update the parameters. We summarize this defense approach in Algorithm~\ref{alg:defend_mis}.

\begin{algorithm}
\caption{Masquerade Defense}
\label{alg:defend_mis}
\begin{algorithmic}[1]
 \REQUIRE Input data $\x_A\in\mathbb{R}^{d_A}$, random binary value $a\in\{0, 1\}$
 \STATE $\z_A \gets \big[\P|\u\big]\begin{bmatrix}\Q\x_A \\ a\end{bmatrix}$
 \STATE Send $\z_A$ to the active party
 \STATE Receive $\frac{\partial L}{\partial \z}$ from the active party \vskip 2pt
 \STATE $\frac{\partial L}{\partial \P} \gets \frac{\partial L}{\partial \z}\x_A^\top\Q^\top$ \vskip 2pt
 \STATE $\frac{\partial L}{\partial \u} \gets a\frac{\partial L}{\partial \z}$ \vskip 2pt
 \STATE $\frac{\partial L}{\partial \Q} \gets \P^\top\frac{\partial L}{\partial \z}\x_A^\top$ \vskip 2pt
 \STATE Update $\P$, $\u$, and $\Q$ by gradient descent
\end{algorithmic}
\end{algorithm}

\begin{remark}
In our description, we implicitly assume that the input matrix $\X_A$ has full rank. If it is not the case, the passive party can simply remove features that can be linearly expressed by others. \peng{Such a preprocess is very common in feature engineering. It won't degrade model performance due to the fact that neural networks are indeed based on linear combinations of features.}
\end{remark}

\section{Evaluation}
\label{sec:eval}
This section describes the experimental evaluation of the proposed attacks and countermeasures. Section~\ref{sec:dataset} describes the dataset and feature characteristics. Section~\ref{sec:setup} illustrates the experiment setup. Section~\ref{sec:absle} and Section~\ref{sec:gauatk} present the attack by solving linear equations, and the attack by solving the linear regression in the presence of Gaussian noise, respectively. Section~\ref{sec:coutermeasure} evaluates the countermeasures.

\begin{figure*}[t]
 \centering
 \begin{subfigure}[b]{\linewidth}
 \centering
 \includegraphics[width=\columnwidth]{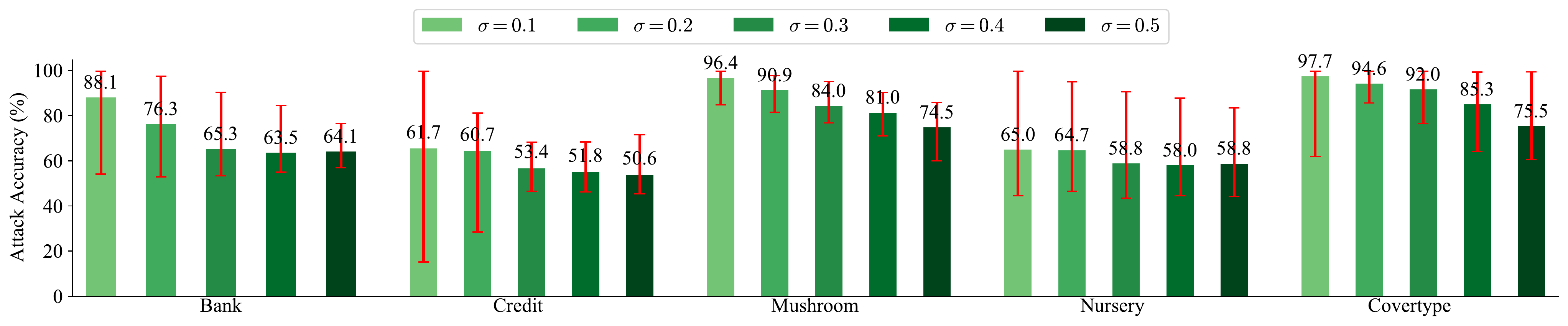}
 \caption{Attack accuracy: varying $\sigma$}
 \label{fig:test}
 \end{subfigure}
 \centering
 \begin{subfigure}[b]{\linewidth}
 \centering
 \includegraphics[width=\columnwidth]{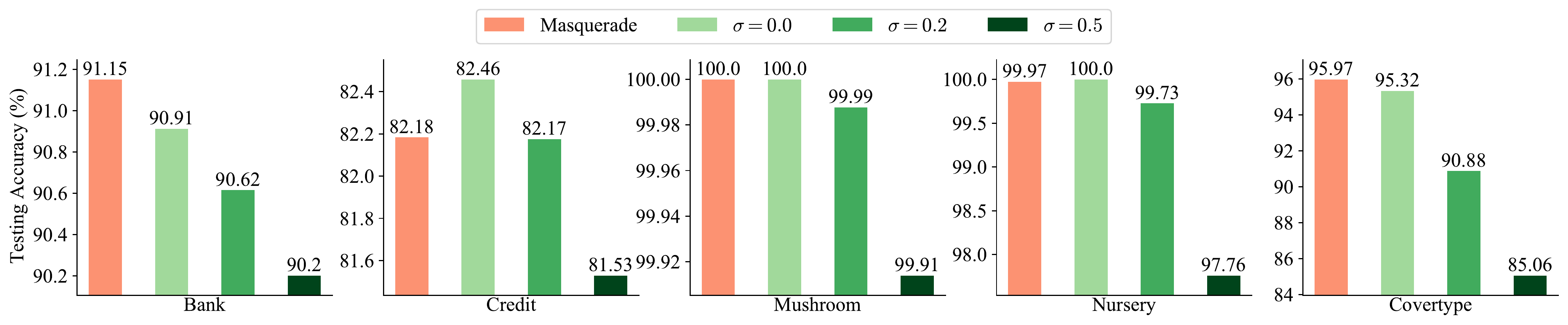}
 \caption{Model accuracy: Masquerade and varying $\sigma$}
 \label{fig:attack}
 \end{subfigure}
 \centering
 \begin{subfigure}[b]{\linewidth}
 \centering
 \includegraphics[width=\columnwidth]{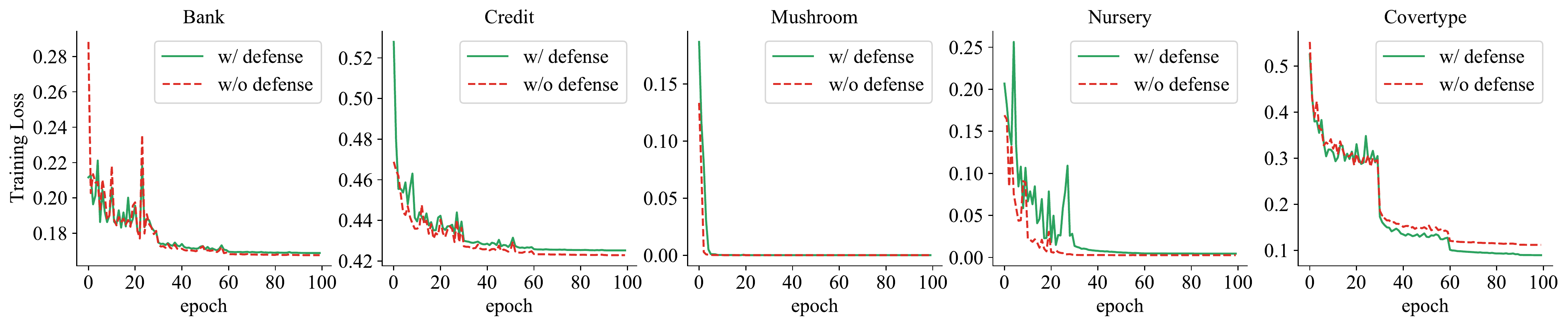}
 \caption{Training loss}
 \label{fig:loss}
 \end{subfigure}
\caption{Attack accuracy, model accuracy, and training loss.}
\label{fig:att}
\end{figure*}
\subsection{Datasets and Models}
\label{sec:dataset}

In our experiments, we use the following five public datasets from the UCI machine learning repository~\cite{Dua:2019 }.\\
\begin{itemize}
 \item Bank\footnote{\url{https://archive.ics.uci.edu/ml/datasets/bank+marketing}} is a dataset that contains information about a bank's $41188$ clients with $20$ attributes~\cite{moro2014data}. The goal is to predict whether the client will subscribe to a term deposit. We split the input features so that the passive party owns $8$ features including a binary feature ``contact''.
 \item Credit\footnote{\url{https://archive.ics.uci.edu/ml/datasets/default+of+credit+card+clients}} is a dataset that consists of $30000$ consumers' credit information where each consumer has $23$ attributes~\cite{yeh2009comparisons}. The task is to predict whether a cardholder will have a default payment. We split the features so that the passive party has $10$ input features, among which ``gender'' is a binary attribute.
 \item Mushroom\footnote{\url{https://archive.ics.uci.edu/ml/datasets/mushroom}} is a dataset that contains descriptions of $8124$ gilled mushrooms with $21$ attributes. The goal is to predict whether a mushroom is poisonous or edible. We split the features so that the passive party has $15$ features including $5$ binary features.
 \item Nursery\footnote{\url{https://archive.ics.uci.edu/ml/datasets/nursery}} dataset consists of $12960$ records of nursery-school applications, where each record contains $8$ features. The target is the final evaluation of every application, which is divided into $5$ levels. We split the features so that the passive party has $6$ features including one binary feature about the financial standing of the families.
 \item Covertype\footnote{\url{https://archive.ics.uci.edu/ml/datasets/covertype}} is a dataset of $581012$ data records with each record contains $54$ attributes extracted from the observation of a certain area. The objective is to determine the forest cover type out of total $7$ types. We split the features so that the passive party owns $10$ features, including $4$ binary features that are converted from a categorical feature by one-hot encoding.
\end{itemize}

Table~\ref{tab:dataset} summarizes the information about the datasets. We use $d_A$ to represent the number of features owned by the passive party. Different model structures (i.e. number of neurons in hidden layers) are employed for different datasets.
\begin{table}[!ht]
\caption{Description of datasets and models used in our experiments.}
\resizebox{\columnwidth}{!}{%
\begin{tabular}{|c|c|c|c|c|c|}
\hline
Dataset   & \#Instances & \#Features & \#Classes & $d_A$ & NN Structure        \\ \hline\hline
bank      & $41188$     & $20$       & $2$       & $8$   & $\{60, 30, 10\}$    \\ \hline
credit    & $30000$     & $23$       & $2$       & $10$  & $\{100, 50, 20\}$   \\ \hline
mushroom  & $8124$      & $21$       & $2$       & $15$  & $\{300, 200, 100\}$ \\ \hline
nursery   & $12960$     & $8$        & $5$       & $6$   & $\{200, 100\}$      \\ \hline
covertype & $581012$    & $54$       & $7$       & $10$  & $\{200, 200, 200\}$ \\ \hline
\end{tabular}%
}
\label{tab:dataset}
\end{table}

\subsection{Experimental Setup}
\label{sec:setup}
We implement the FL algorithm with our attack and defense methods in PyTorch~\cite{paszke2019pytorch}. The experiments were conducted on a computer equipped with Intel(R) Xeon(R) CPU @ 2.20GHz and 16GB RAM, running Ubuntu 20.04.4 LTS.

For each dataset, we randomly split it so that $90\%$ of the data records are used for training and the remaining $10\%$ are for testing. We optimize the neural network for $100$ epochs using SGD with momentum, where the momentum is set as $0.9$. The base learning rate is $0.1$ and we reduce the learning rate by a factor of $10$ after $30$, $60$, and $90$ epochs, respectively. We adopt cross entropy loss with a $10^{-4}$ weight decay as the objective function. Such a hyper-parameter setting is commonly used in training neural networks (e.g.,~\cite{yang2021simam,chen2021distilling}). Each experiment is repeated $20$ times.

\subsection{Attack by Solving Linear Equations}
\label{sec:absle}
\subsubsection{Effectiveness}
We first conduct experiments on training neural networks without adding noise. After training, we first extract intermediate results by feeding the entire dataset into the model. Then we perform Algorithm~\ref{alg:attack_equ} to recover the input binary features. Our experimental results show that we can recover the input binary features with $100\%$ accuracy.

For bank, credit, and nursery dataset, there is only one input binary feature on the passive side. In our experiments, our attack algorithm outputs the only binary feature as expected. For the mushroom dataset, which has $5$ binary features of the passive party, the attack algorithm produces all of them.

For covertype dataset, the passive party has a categorical feature, which is converted to $4$ binary features by one-hot encoding. The $4$ binary vectors don't overlap at any coordinate. Thus, in this case, the element-wise sum of any subset of the $4$ binary vectors is also a binary vector in the column space of the intermediate output. As a result, Algorithm~\ref{alg:attack_equ} will be able to identify it as well. In our experiment, the attack algorithm successfully finds the binary features along with their combinations (totally $2^4 - 1 = 15$ binary vectors, corresponding to all non-empty subsets of the $4$ binary features), from which we can reconstruct the input categorical feature.

\subsection{Attack by Solving Linear Regression in the Presence of Gaussian Noise}
\label{sec:gauatk}
\subsubsection{Effectiveness}
If the passive party adds some noise before transmitting the intermediate results, an attack by solving linear equations becomes ineffective since the linear equations no longer hold. To handle this, we propose Algorithm~\ref{alg:attack_reg}, which relies on solving linear regression.

We test this attack over five datasets with Gaussian noise added to the intermediate results. For each dataset, we vary $\sigma$ from $0.1$ to $0.5$. Similarly, we run Algorithm~\ref{alg:attack_reg} on the intermediate results of the five datasets.

To measure the effectiveness of our attack, we compared the output of our attack algorithm with the true input features. We introduce the concept of attack accuracy defined as,
\[\max_{\x\in\text{ input features}}\frac{1}{n}\sum_{i=1}^n\mathbb{I}(\x^*_i = \x_i),\]
where $x^*$ is the output of Algorithm~\ref{alg:attack_reg} and $\mathbb{I}$ is the indicator function. That is, we consider all input features and find the one that has the most coordinates that match the algorithm's output. For one-hot features, we also take their combinations into account. Clearly, a high attack accuracy indicates the effectiveness of the attack while a low attack accuracy means the protection is successful.

In the experiments we set $r = d_A +1$ in Algorithm~\ref{alg:attack_reg}. Since it is a randomized algorithm, we run it for $20$ times and choose the solution that has the minimum error. We present the range of attack accuracy under different $\sigma$ in Figure~\ref{fig:attack}. The error bars represent the maximum and minimum attack accuracy in our experiments. From the results, we can observe a descending trend of attack accuracy as we increase $\sigma$, indicating that larger noise can provide better protection against the attack.

\subsubsection{Performance}
\label{sec:performance}
When Gaussian noise is introduced, the performance of the neural network may be affected. In addition to attack accuracy, we also evaluated the effect on model accuracy when we added noise to the intermediate output. We measure the model accuracy under different levels of noise (i.e. different $\sigma$) on every dataset. The results are plotted in Figure~\ref{fig:test}, which show that the model accuracy drops as $\sigma$ increases. Therefore, although masking the intermediate output with Gaussian noise can provide a certain level of protection, it also sacrifices model quality.

\subsection{Countermeasure}
\label{sec:coutermeasure}
The masquerade defense method we propose targets to mislead the attacker to a randomly generated binary feature. Thus, the attacker will only find the fabricated binary features after performing the attack so that the true input features will be properly protected.

We evaluated our masquerade defense over the five datasets. In our experiments, we compared the fabricated feature and the solution produced by the attack algorithm. We find that it only outputs a single solution in the fabricated feature as expected. This indicates that our defense effectively misguides the attacker to the randomly generated binary feature, and therefore protects true input features.

\peng{We also wonder how this affects model performance because our defense restricts the rank of the weight matrix.} We measure the model performance and the results are shown in Figure~\ref{fig:test}, together with defense by adding Gaussian noise. Compared to the defense mechanism by adding Gaussian noise, our masquerade method has much higher model accuracy. Moreover, We plot the training loss curve with and without defense in Figure~\ref{fig:loss}. The results illustrate that the two training loss curves are close, indicating that our defense method enjoys nearly no loss in model performance.

\section{Discussion}
\label{sec:discuss}
\peng{The attack mechanisms proposed are applicable under a variety of scenarios. For instance, one might think of simply adding a bias term or representing a binary feature by other values instead of $0$ and $1$ to invalidate the attack. However, we can pick one row of $\Z_A$ and subtract this row from other rows to eliminate the bias term. Then we can apply our attack to the resulting matrix.}

Noticeably, in this work, we focus on the attack and countermeasure only for two-party VFL. However, our attack methods can be easily extended to the multi-party scenario. More specifically, if the passive parties send their intermediate results to the active party directly, we can perform the attack on the output sent from each passive party exactly the same way as in the two-party scenario. When the passive parties use secure aggregation to sum their intermediate results, our attack algorithm is still applicable, with the dimension $d_A$ replaced by the total dimension of input features owned by all passive parties. Even they adopt the protection proposed in~\cite{fu2022BlindFL}, the sum of intermediate results of all parties (including the active party) is still exposed to the active party. Our proposed attack algorithm can be applied to the sum to extract binary features as long as the total dimension is still within reach.

Moving beyond binary features, for instance categorical features, a common way to do feature engineering is one-hot encoding. In this case, our attack can find the converted binary features (and their sums), from which we can recover the categorical features. In the case that a categorical feature is transformed into a single multi-valued feature, our attack still works if the attacker knows what the values are.

A limitation is that our attack methods are based on the fact that the cut layer is the input layer. If we cut at the other layers (e.g., the second layer), our attack algorithms cannot work because the linearity the algorithms rely on no longer holds after the nonlinear activation functions are introduced.

Another limitation is that, although the masquerade defense devised can successfully mislead the attacker to a fabricated feature, it does not ensure absolute security. For instance, the attacker can adaptively change its attack method to recover input features. After attacking the fabricated binary vector, it retains all coordinates that are $0$'s and discards all $1$'s. Thus, it obtains a $(d_A - 1)$-dimensional subspace, which is still vulnerable to our attack methods. How to protect the data against any kind of attack while keeping the model performance remains open.

\section{Related Work}
\PHM{Data reconstruction attacks.} Our proposed feature attack is one type of general \emph{data reconstruction attacks}, which seek to recover the private input data. In VFL, there are two categories of data reconstruction attacks: (1) feature inference attacks, where an active party attempts to recover a passive party's input features; and (2) label inference attacks, where a passive party tries to discover the active party's labels.

\PHM{Feature attacks.} This implies an attack on a passive party's input features. Since an active party possesses far more information, thus, it is in a more advantageous position for such attacks. In~\cite{weng2021Privacy}, Weng \emph{et al.} devised a reverse multiplication attack method against logistic regression with the assistance of a corrupted third-party coordinator and a reverse sum attack method against XGBoost by encoding magic numbers in the gradients. In~\cite{luo2021Feature}, Luo \emph{et al.} designed an equality solving attack for linear regression models, a path restriction attack for decision tree models, and a generative regression network for attacking more complex models. This work adopts a white-box setting, which requires an active party to know the entire model weights including the passive party's local model. In~\cite{jiang2022Comprehensive}, Jiang \emph{et al.} proposed a gradient-based inversion attack, which can recover a passive party's input under both white-box and black-box settings with the assistance of a set of auxiliary data used in training.

\PHM{Label attacks.} In~\cite{fu2022Label}, Fu \emph{et al.} proposed a label inference attack based on the semi-supervised learning technique, which can recover an active party's labels using its local bottom model and a small set of auxiliary data. In~\cite{liu2022Batch}, Liu \emph{et al.} presented a gradient inversion attack, which can infer the labels from batch-averaged gradients when the top model is a softmax function on the sum of intermediate results and the loss function is cross entropy. In~\cite{li2022Label}, Li \emph{et al.} considered a two-party split learning scenario and designed two attack mechanisms to extract labels from the norm and direction of intermediate gradients.

\section{Conclusion}

In this paper, we take the initiative to study the feature security problem of DNN training in VFL. We first prove that feature attacks are not possible when the attacker has zero knowledge of the dataset. We then focus on client data with binary features, and show that unless the feature space is exceedingly large, we can precisely reconstruct the binary features in practice with a robust search-based attack algorithm. We proceed to present a defense mechanism that overcomes such binary feature vulnerabilities by misleading the adversary to search for fabricated features. Our experiments show that our feature reconstruction attack is extremely effective against VFL on realistic DNN training tasks. Yet, the defense method proposed can effectively thwart the attack with a negligible loss in model accuracy.



\bibliographystyle{IEEEtranS}
\bibliography{references}

\end{document}